\newcommand{\CLASSINPUTinnersidemargin}{0.75in} 
\newcommand{\CLASSINPUToutersidemargin}{0.75in} 
\newcommand{\CLASSINPUTtoptextmargin}{0.75in} 
\newcommand{\CLASSINPUTbottomtextmargin}{0.75in}

\documentclass[conference, letterpaper]{IEEEtran}
\IEEEoverridecommandlockouts    

\usepackage[normalem]{ulem}	                        
\usepackage[table,usenames,dvipsnames]{xcolor}      
\usepackage{extarrows}                              
\usepackage[inline]{enumitem}
\usepackage{cite}

\usepackage{amsmath,amssymb,amsfonts,amsthm,dsfont} 
\usepackage{mathtools, bm}
\usepackage{algorithm,algorithmicx,listings}        
\usepackage[noend]{algpseudocode}			              

\usepackage{graphicx}
\usepackage{tabularx}
\usepackage[font={small}]{caption}   
\usepackage{subcaption}
\usepackage[breaklinks=true, colorlinks, bookmarks=true, citecolor=Black, urlcolor=Violet,linkcolor=Black]{hyperref}

\def\argmin{\mathop{\arg\min}\limits}	%
\newcommand{\crl}[1]{\mathopen{}\left\{#1\right\}\mathclose{}}

\newtheorem{theorem}{Theorem}

\newtheorem{corollary}{Corollary}[theorem]

\newtheorem*{assumption*}{Assumption}
\newtheorem{remark}{Remark}
\newtheorem*{problem*}{Problem}

\def\thetitle{Dense 3-D Mapping with Spatial Correlation via Gaussian Filtering}
\def\theauthor{Ke Sun, Kelsey Saulnier, Nikolay Atanasov, George J. Pappas, and Vijay Kumar}

\hypersetup{
  pdfauthor={\theauthor},%
  pdftitle={\thetitle},%
  pdfsubject={\thetitle},%
  pdfkeywords={information planning, controlled sensing, sensor management, active slam, robotics, mutual information, differential entropy, model predictive control, receding horizon, square root information filter, simultaneous localization and mapping}
}

\begin{document}
\title{\vspace{2ex}\LARGE \bf \thetitle}%
\author{\theauthor
  \thanks{The authors are with GRASP Lab, University of Pennsylvania, Philadelphia, PA 19104, USA, {\tt\small\{sunke, saulnier, atanasov, pappasg, kumar\}@seas.upenn.edu}.}%
  \thanks{We gratefully acknowledge support by TerraSwarm, one of six centers of STARnet, a Semiconductor Research Corporation program sponsored by MARCO and DARPA and the following grants: NSF-IIA-1028009, ARL MAST-CTA W911NF-08-2-0004, ARL RCTA W911NF-10-2-0016, ONR-N000141310778, NSF-DGE-0966142, NSF-IIS-1317788, NSF-IIP-1439681, and NSF-IIS-1426840.}
}
\maketitle

\begin{abstract}
Constructing an occupancy representation of the environment is a fundamental problem for robot autonomy. Many accurate and efficient methods exist that address this problem but most assume that the occupancy states of different elements in the map representation are statistically independent. The focus of this paper is to provide a model that captures correlation of the occupancy of map elements. Correlation is important not only for improved accuracy but also for quantifying uncertainty in the map and for planning autonomous mapping trajectories based on the correlation among known and unknown areas. Recent work proposes Gaussian Process (GP) regression to capture covariance information and enable resolution-free occupancy estimation. The drawback of techniques based on GP regression (or classification) is that the computation complexity scales cubically with the length of the measurement history. Our main contribution is a new approach for occupancy mapping that models the binary nature of occupancy measurements precisely, via a Bernoulli distribution, and provides an efficient approximation of GP classification with complexity that does not scale with time. We prove that the error between the estimates provided by our method and those provided by GP classification is negligible. The proposed method is evaluated using both simulated data and real data collected using a Velodyne Puck 3-D range sensor.

\end{abstract}

\section{Introduction}
\label{sec:introduction}
 
Constructing a geometric representation of the environment is a fundamental problem in robotics because it is essential for localization, reliable navigation, as well as occlusion and collision checking in tasks such as object recognition and manipulation. Of course, geometric mapping has been studied extensively and there exist many accurate and efficient methods~\cite{occgrid,octomap,kinfu,elastic_fusion,rgbd-slam,Newcombe_DenseVisualSLAM_Phd14,lsd-slam}. While accuracy and efficiency are important requirements, in this paper we focus on an additional requirement that many existing methods do not address. In particular, we focus on maintaining information about covariance of different elements in the map representation. For example, occupancy grid mapping~\cite[Ch.9]{occgrid,ProbabilisticRoboticsBook}, one of the early, very successful approaches to geometric mapping, represents the environment as a regular grid in which each cell is a Bernoulli random variable that captures the probability of the cell being occupied or free. Occupancy grid mapping relies on the assumption that the occupancy states of different grid cells are statistically independent and sensor data is correlated only with directly observed cells. We claim that maintaining a model of occupancy dependence among the cells may improve the accuracy of the resulting maps. However, our main motivation for considering such a model with covariance information is that the correlation between known and unknown areas in the map enables planning informative robot trajectories for autonomous mapping~\cite{active_gp_conference,active_gpmap,Charrow-RSS-15}. More precisely, covariance statistics can be used to quantify the current map uncertainty (e.g., the via Shannon's entropy~\cite{Krause_PhD08,Atanasov_PhD15,Charrow_PhD15}) and to infer the probable occupancy of unobserved locations based on the occupancy of already observed cells.

\begin{figure}[t]
\includegraphics[width=0.49\linewidth]{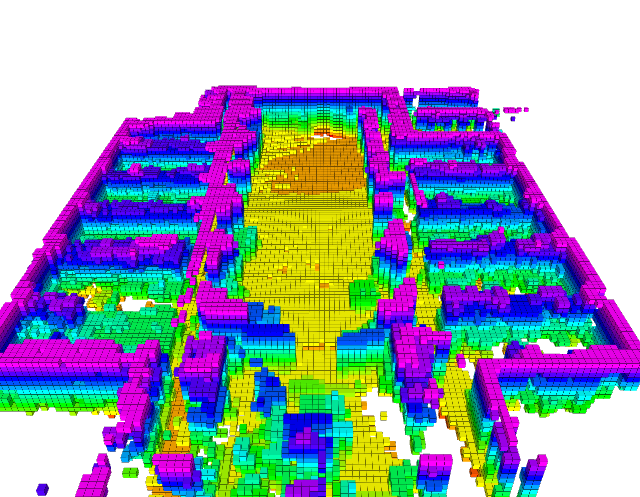}
\includegraphics[width=0.48\linewidth]{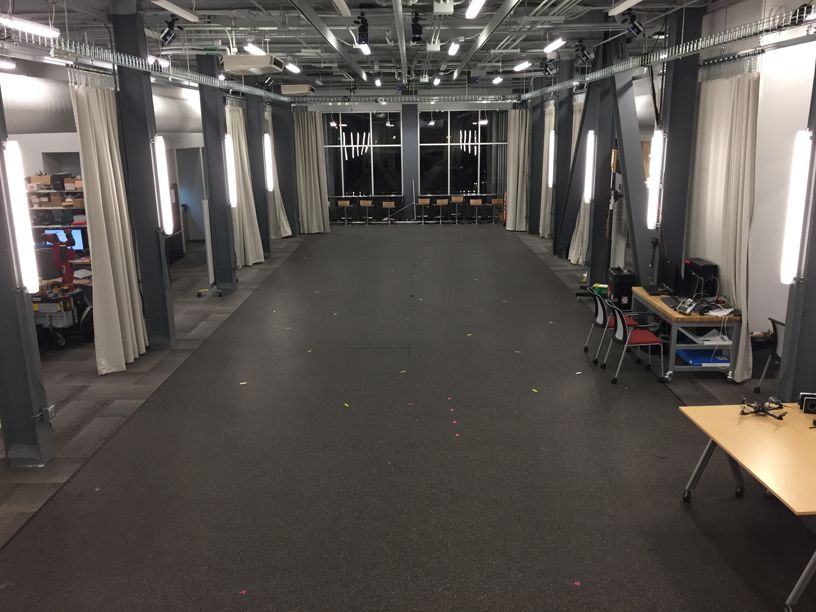}
\caption{A dense 3-D map (left) of a lab environment (right) produced using the occupancy grid filter, one new approach which approximates GP classification accurately with computational complexity that does not scale with time.}
\label{fig:advertisement  figure}
\end{figure}

The above has been recognized by several recent works~\cite{gpom_conf,gpom,gpmap_integralkernel,hilbert_maps,area_kernels,Kim_BCM,Kim_TRO,Kim2015,Wang_GPRegressionMapping} that propose the use of Gaussian Processes (GPs)~\cite{Rasmussen_2006} for occupancy mapping. In addition to capturing covariance information via a kernel function, these models provide resolution-free mapping in the form of a continuous occupancy function. The challenge for GP methods is to deal efficiently with the hybrid nature of the problem -- discrete measurements (occupied or free) and continuous occupancy function. O'Callaghan et al.~\cite{gpom_conf,gpom,gpmap_integralkernel} proposed a probabilistic least-squares classification method in which a probit\footnote{The probit functions is defined as the culmulative distribution function of a standard normal, i.e. $\Phi(x;0, 1) = \int_{-\infty}^x \phi(u;0,1)du$, where $\phi(u;0,1)$ is the standard normal pdf.} function is used to convert a latent function estimated via GP regression to binary occupancy predictions. A more accurate but computationally demanding mapping method is to use GP classification~\cite[Ch.3.3]{Rasmussen_2006}, in which the latent function posterior is updated from the discrete occupancy measurements. Since the measurement likelihood in this case is not Gaussian, the posterior distribution cannot be analytically determined and has to be approximated with iterative methods. The major drawback of GP regression and classification is that their computational complexity is cubic in the number of measurements. In an online mapping setting where the number of measurements grows with time, applicability of these methods is limited.


Many techniques~\cite{GP_direct,sparseGP,Anderson2015} focus on improving the computational complexity of GPs in general, relying on sparse kernels and matrix factorization to enable efficient inversion of the covariance matrix (resulting from evaluating the kernel at the training examples). Specific to GP mapping, Kim and Kim~\cite{Kim_BCM,Kim_TRO,Kim2015} proposed the use of a sparse kernel and Bayesian Committee Machines to perform small GP regressions with subsets of training data before fusing them into the full map. Similarly, Wang and Englot~\cite{Wang_GPRegressionMapping} partition their sensor data amongst several GP regressions and then use Baysian Committee Machines to first combine the sensor-level regressions and then again to fuse the resulting regression with the map. The speed of this method is further improved with a new data-structure, the test-data octree, which improves access time and only considers correlations between grid cells within a specified range. Despite its advantages, this method does not capture the discrete nature of the measurements in its model and ignores some cross correlations in the grid.



Our key insight, compared to the existing work which uses GPs for metric mapping, is that the values of the latent function at a discrete set of points can be recovered exactly by a Kalman filter~\cite{kf_gp_connection}. More precisely, the state of the Kalman filter can capture the continuous values of the latent occupancy function evaluated at each cell of a discrete map, while a probit observation model can convert these latent values into boolean measurements indicating the occupancy of the corresponding cell. The significance of this observation is that a Kalman filter can avoid the cubic scaling with the number of data points (observations of free/occupied space), while providing essentially the same estimates of the mean and covariance of the map. Our second idea is that, although the use of a probit measurement model makes the posterior distribution of the filter non-Gaussian, the posterior can be approximated via a Gaussian based on a single forward pass through the measurement data. We prove that this method results in an efficient approximation of GP classification~\cite[Ch.3.6]{Rasmussen_2006}. We name the proposed approach an Occupancy Grid Filter (OGF) to emphasize that these two insights were inspired by filtering techniques. Our work makes the following contributions:
\begin{itemize}[nosep]
  \item We prove that the OGF is an accurate approximation of GP classification on a grid and show that in practice the classification error is negligible.
  \item The computational complexity of the OGF is constant in time while GP methods scale cubically with time as the number of measurements grows. This property makes the proposed occupancy grid filter suitable for online applications.
\end{itemize}

\section{Problem Formulation}
\label{sec:problem formulation}
Consider a robot whose task is to build a geometric map of a static environment. Let $\mathcal{S} \subset \mathbb{R}^3$ be a bounded connected set representing the environment and let $o:\mathcal{S}\mapsto \crl{-1,1}$ be a function such that $o(\mathbf{s}) \coloneqq -1$, if $\mathbf{s} \in \mathcal{S}$ is unoccupied, and $o(\mathbf{s}) \coloneqq 1$, otherwise. We restrict attention to occupancy grid mapping~\cite[Ch.9]{ProbabilisticRoboticsBook}. More precisely, we consider a tessellation of $\mathcal{S}$ with an associated lattice $\mathcal{X} \subset \mathcal{S}$ and estimate the map occupancy $o(\mathbf{x})$ only at the locations $\mathbf{x} \in \mathcal{X}$. Suppose that the robot collects a sequence of occupancy measurements $Z_{0:t}$, e.g., via a laser scanner or a depth camera, along a known trajectory $(\mathbf{s}_{0:t},R_{0:t}) \subset \mathbb{R}^3 \rtimes \text{SO}(3)$. Each occupancy measurement $Z_t$ is a matrix containing the 3-D positions of the cells $\mathbf{x}$ observed by the sensor and their occupancy $o(\mathbf{x})$, i.e., $Z_t \coloneqq (X_t,\mathbf{y}_t)$, $\mathbf{x}_{i,t} \in \mathcal{X} \cap \text{FoV}(\mathbf{s}_t,R_t)$, $y_{i,t} \coloneqq o(\mathbf{x}_{i,t})$, and $\text{FoV}(\mathbf{s}_t,R_t) \subset \mathcal{S}$ is the sensor's field of view. 

\begin{problem*}
Given occupancy measurements $Z_{0:t}$, construct an estimate $\hat{o}_t \in \crl{-1,0,1}^{|\mathcal{X}|}$ of the true environment occupancy $\crl{o(\mathbf{x}) | \mathbf{x} \in \mathcal{X}}$.
\end{problem*}

Note that one more state, $0$, is added in order to allow the case that the occupancy of a cell is unknown.

\section{Gaussian Process Classification}
\label{sec:gaussian process}
In this section, we give a brief overview of GPs and how they are applied to the occupancy mapping problem. As defined in~\cite{Rasmussen_2006}, a GP is a collection of random variables, any finite number of which have a joint Gaussian distribution, represented as, incorporate 
\begin{equation*}
f(\mathbf{s}) \sim 
\mathcal{GP}\left(\hat{f}(\mathbf{s}), k(\mathbf{s}, \mathbf{s}')\right)
\end{equation*}
where $\hat{f}(\mathbf{s}) = \mathbb{E}(f(\mathbf{s}))$ is the mean of $f(\mathbf{s})$, and $k(\mathbf{s}, \mathbf{s}') = \mathbb{E}((f(\mathbf{s})-\hat{f}(\mathbf{s}))(f(\mathbf{s}')-\hat{f}(\mathbf{s}')))$ is a kernel function used to construct the covariance matrix.

In a binary classification problem, $f(\mathbf{s})$ cannot directly model the class labels because of their discontinuous nature. Instead, $f(\mathbf{s})$ is treated as a latent function with GP prior, which is then squashed through the probit function to model the likelihood of a measurement, i.e. $p(y=1|\mathbf{s})$ or $p(y=-1|\mathbf{s})$. The following two steps are required in a typical application of binary GP classification to mapping~\cite{gpom}.

\subsection{Training}
In the training step, the posterior
\begin{equation}
p(\mathbf{f}_t|Z_{0:t}) = 
\frac{p(\mathbf{y}_{0:t}|\mathbf{f}_t) p(\mathbf{f}_t|S_{0:t})}{p(\mathbf{y}_{0:t}|S_{0:t})}
\label{eq:posterior of latent values}
\end{equation}
is used to estimate the latent values, $\mathbf{f}_t\coloneqq f_{0:t}$, of all the captured samples. Note the slight abuse of notation, $Z_{0:t}$ in~\eqref{eq:posterior of latent values} should be more accurately represented as $Z_{0:t}\coloneqq (S_{0:t}, \mathbf{y}_{0:t})$ instead of $(X_{0:t}, \mathbf{y}_{0:t})$, i.e. the training samples can be taken at any points $\mathbf{s}\in \mathcal{S}$. In~\eqref{eq:posterior of latent values}, $p(\mathbf{f}_t|S_{0:t})\sim\mathcal{N}(0, K)$ is the prior of the latent values of all training samples, where $K$ is a covariance matrix constructed using a pre-defined kernel, $p(\mathbf{y}_{0:t}|S_{0:t})$ is a normalization factor independent of $\mathbf{f}_t$, and $p(\mathbf{y}_{0:t}|\mathbf{f}_t)$ is the joint likelihood function which can be represented as,
\begin{equation}
\label{eq:joint liklihood function}
p(\mathbf{y}_{0:t}|\mathbf{f}_{t}) = 
\prod_i p(y_i|f_i) = \prod_i \Phi(f_i;0,1)
\end{equation}
because the measurements are independent given the latent values. The posterior pdf in~\eqref{eq:posterior of latent values} cannot be obtained in closed form since $p(\mathbf{y}_{0:t}|\mathbf{f}_t)$ is not of normal distribution. There exist different methods to approximate the posterior. The Laplace approximation~\cite[Ch.3]{Rasmussen_2006} finds the mode of $p(\mathbf{f}_t|Z_{0:t})$ and the local negative inverse Hessian to construct a normal distribution that approximates $p(\mathbf{f}_t|Z_{0:t})$. Expectation Propagation (EP)~\cite[Ch.3]{Rasmussen_2006} replaces each term in~\eqref{eq:joint liklihood function} with a (unnormalized) normal distribution, and tries to approximate the true posterior with product of normals in an iterative manner. More details on EP can be found in Sec.~\ref{sec:comparison}.

\subsection{Prediction}
With the estimated latent values of the training samples, $\hat{\mathbf{f}}_t$, the latent values, $\mathbf{m}_t$, of a query position in $\mathcal{X}$ can be estimated as,
\begin{equation*}
\hat{\mathbf{m}}_t = K^* \cdot K^{-1}\hat{\mathbf{f}}_t
\end{equation*}
where $K^*$ represents the correlation between $\mathbf{m}$ and $\mathbf{f}$, which is constructed using the kernel function with query positions and the positions of the training samples. By squashing $\hat{\mathbf{m}}_t$ through the probit function, we are able to recover the occupancy of the grid,
\begin{equation}
\label{eq:convert to occupancy status}
\hat{o}(\mathbf{x}_i) = 
\begin{cases}
1,  & \Phi(m_i; 0, 1) > r_o \\
-1, & \Phi(m_i; 0, 1) < r_f \\
0,  & \text{otherwise}
\end{cases}
, \quad
\forall\mathbf{x}_i \in \mathcal{X}
\end{equation}
where $r_o$ and $r_f$ are the thresholds for determining occupied and free cells respectively.

\section{Occupancy Grid Filter}
\label{sec:context aware filter}
In this section, we propose a new solution to the mapping problem which we call the Occupancy Grid Filter. This method takes into account the spatial correlation and discontinuous nature of occupancy and, at the same time, has lower computational cost than GP classification methods. The OGF arises from two key ideas. First, instead of creating a continuous map which is then evaluated at query points, we preselect the query points and only maintain latent function values at those points. Second, we note that the training step of GP classification described in Sec.~\ref{sec:gaussian process} is a batch optimization which must be re-evaluated when new measurements become available. We replace this with an incremental, Kalman-filter-like approach which greatly reduces the computational complexity. 
By defining a grid and abstracting each cell with its center position $\mathbf{x}_i$, we only need keep track of the latent values $\mathbf{m}$ of the query positions $\mathcal{X}\coloneqq \{\mathbf{x}_i\}$, where $\mathbf{m}$ can be thought as the state of the map. Following the idea of GP classification, the prior of the state is $\mathbf{m} \sim \mathcal{N}(\hat{\mathbf{m}}_0, \Sigma_0)$, where $\hat{\mathbf{m}}_0=\mathbf{0}$ and $\Sigma_0$ is constructed by applying a kernel function on $\mathcal{X}$, i.e. $\Sigma_{0,ij}\coloneqq k(\mathbf{x}_i,\mathbf{x}_j)$. Note by initializing the prior covariance with the kernel we are capturing the same correlation which would be captured by the associated GP. With the assumption that the measurement $\mathbf{z}_i\coloneqq(\mathbf{x}_i^\top, y_i)$ is within the pre-defined tessellation $\mathcal{X}$, the binary measurement model can be represented as
\begin{equation*}
y_i = h(\mathbf{x}_i) \in \{-1, 1\}
\end{equation*}
with the Bernoulli distribution,
\begin{equation*}
p(y_i|\mathbf{x}_i, m_i) = \Phi(y_i m_i;0, 1).
\end{equation*}
Assume that the distribution of the map follows $\mathbf{m} \sim \mathcal{N}(\hat{\mathbf{m}}_t, \Sigma_t)$ at time $t$. With a measurement for the $i^{\text{th}}$ cell at time $t+1$, the posterior distribution of the map can be computed based on Bayes' rule, i.e.,
\begin{equation*}
p(\mathbf{m}|Z_{0:t}) 
= \frac{\phi(\mathbf{m};\hat{\mathbf{m}}_t, \Sigma_t) 
\Phi(y_im_i;0, 1)}
{p(\mathbf{z}_{i,t+1})}
\end{equation*}
Since the likelihood function is a cumulative normal function, the posterior of $\mathbf{m}$ at $t+1$ is no longer normal. To ensure the map follows a normal pdf, the posterior is estimated by minimizing the following Kullback-Leibler (KL) divergence\footnote{The $\text{KL}$ operator is defined as $\text{KL}(p||q) = \int p(x)\log(p(x)/q(x))dx$. Note that the $\text{KL}$ operator is not associative. The order in~\eqref{eq:KL divergence for EP} is chosen so that~\eqref{eq:KL divergence for EP} can be solved analytically.}
\begin{equation}
\label{eq:KL divergence for context-aware filter}
\argmin_{\hat{\mathbf{m}}_{t+1}, \Sigma_{t+1}} 
\text{KL}
\Big(
\phi(\mathbf{m}; \hat{\mathbf{m}}_{t+1}, \Sigma_{t+1}) \,\big\|\,
p(\mathbf{m}|Z_{0:t})
\Big)
\end{equation}
The same objective function is used in~\cite{Ivanov_2015} to enable the context-aware filter to incorporate binary measurements. 
We adapt the formulation of the discrete measurement update from the context-aware filter to solve~\eqref{eq:KL divergence for context-aware filter} for the mapping problem. The mean and covariance estimation for the map at $t+1$ is
\begin{equation}
\label{eq:context aware filter discrete udpate}
\begin{gathered}
\hat{\mathbf{m}}_{t+1} = 
\hat{\mathbf{m}}_t + y_i \cdot \Sigma_t \mathbf{v}_i 
\frac{\phi\left(y_i \cdot \frac{\mathbf{v}_i^\top \hat{\mathbf{m}}_t}{\sqrt{\mathbf{v}_i^\top \Sigma_t \mathbf{v}_i+1}}; 0, 1\right)}
{\eta_{t+1} \sqrt{\mathbf{v}_i^\top \Sigma_t \mathbf{v}_i+1}} \\
\Sigma_{t+1} = 
\Sigma_t-\left(\hat{\mathbf{m}}_{t+1}-\hat{\mathbf{m}}_t\right)\left(\hat{\mathbf{m}}_{t+1}-\hat{\mathbf{m}}_t\right)^\top - \\
y_i \cdot \Sigma_t \mathbf{v}_i \mathbf{v}_i^\top \Sigma_t^\top 
\frac{\phi\left(y_i \cdot \frac{\mathbf{v}_i^\top \hat{\mathbf{m}}_t}{\sqrt{\mathbf{v}_i^\top \Sigma_t \mathbf{v}_i+1}}; 0, 1\right) \left(\mathbf{v}_i^\top \hat{\mathbf{m}}_t\right)}
{\eta_{t+1}\left(\mathbf{v}_i^\top \Sigma_t \mathbf{v}_i+1\right)^{3/2}}
\end{gathered}
\end{equation}
where $\mathbf{v}_i$ is a selection vector containing all $0$s but a single $1$ corresponding to the $i^\text{th}$ cell, $\eta_{t+1}$ is the normalization factor,
\begin{equation*}
\eta_{t+1} = 
p(\mathbf{z}_{i,t+1}) = 
\Phi\left(y_i \cdot 
\frac{\mathbf{v}_i^\top \hat{\mathbf{m}}_t}
{\sqrt{\mathbf{v}_i^\top \Sigma_t \mathbf{v}_i+1}}; 0, 1\right)
\end{equation*}
The occupancy of each cell can then be recovered through~\eqref{eq:convert to occupancy status}. Unlike GP classification, the occupancy grid filter merges both training and prediction steps, and incrementally updates the map only keeping track of the latent values on the grid.

\section{Theoretical Properties of OGF}
\label{sec:comparison}
A mentioned in Sec.~\ref{sec:gaussian process}, GP classification cannot be solved in closed form. In this section we briefly describe one of the conventional iterative solutions, EP, in the context of the mapping problem. We then show that the occupancy grid filter introduced in Sec.~\ref{sec:context aware filter} is a streamwise approximation of EP and can solve the occupancy mapping problem with computation complexity linear in the number of measurements at each time step and quadratic in the size of the map. 
The idea of EP is to approximate the posterior distribution of $\mathbf{m}$, $p(\mathbf{m}|Z_{0:t})$, by
\begin{equation*}
q(\mathbf{m}|Z_{0:t}) = 
\frac{1}{\eta_{ep}} 
p(\mathbf{m}|X_{0:t}) 
\prod_i t_i(m_{i}|\tilde{\eta}_i, \tilde{\mu}_i, \tilde{\sigma}_i^2),
\end{equation*}
where $\eta_{ep}$ is a normalization factor and
\begin{equation*}
t_i(m_{i}|\tilde{\eta}_i, \tilde{\mu}_i, \tilde{\sigma}_i^2) = 
\tilde{\eta}_i \phi(m_{i}; \tilde{\mu}_i, \tilde{\sigma}_i^2),
\end{equation*}
is an unnormalized normal pdf. For convenience, we follow the naming convention in~\cite[Ch.3]{Rasmussen_2006} calling $\tilde{\mu}$, $\tilde{\sigma}$, and $\tilde{\eta}$ site parameters. Conceptually, EP approximates the original posterior distribution with a product of normals, which itself is also a normal pdf. The site parameters are estimated iteratively, and the final approximation, $q(\mathbf{m}|Z_{0:t})$, can be easily recovered by combining each term after convergence. At each iteration, the cavity distribution of $m_i$, defined as
\begin{equation*}
q_{-i}(m_{i}) 
= \int_{j, j\neq i} q\cdot t^{-1}_i d m_{j}
\sim \mathcal{N}(\mu_{-i}, \sigma_{-i}^2),
\end{equation*}
is found with $\mu_{-i}$ and $\sigma_{-i}$ called cavity parameters. Then the site parameters for $t_i$ are updated by minimizing the KL divergence between $q_{-i}(m_{i})t_i$ and $q_{-i}(m_{i})\Phi(y_i m_{i})$, i.e.
\begin{equation}
\label{eq:KL divergence for EP}
\argmin_{\tilde{\eta}_i, \tilde{\mu}_i, \tilde{\sigma}_i}
\text{KL} 
\Big(q_{-i}(m_{i})t_i \,\big\|\, 
q_{-i}(m_{i})\Phi(y_i m_{i};0,1)\Big)
\end{equation}
Since $q_{-i}(m_{i})t_i$ is in the exponential family, (\ref{eq:KL divergence for EP}) can be solved by matching the first two moments of the distributions. As given in~\cite[Ch.3]{Rasmussen_2006}, the solution to (\ref{eq:KL divergence for EP}) is,
\begin{equation}
\label{eq:expectation propagation solution}
\begin{gathered}
z_i = \frac{y_i\mu_{-i}}{\sqrt{1+\sigma_{-i}^2}}, 
\quad
\hat{\eta}_i = \Phi(z_i;0,1) \\
\hat{\mu}_i = 
\mu_{-i} + 
\frac{y_i\sigma^2_{-i}\phi(z_i;0,1)}
{\hat{\eta}_i\sqrt{1+\sigma_{-i}^2}} \\ 
\hat{\sigma}^2_i = 
\sigma^2_{-i} - 
\frac{\sigma^4_{-i}\phi(z_i;0,1)}{(1+\sigma_{-i}^2)\hat{\eta}_i} 
\left(z_i+\frac{\phi(z_i;0,1)}{\hat{\eta}_i}\right)
\end{gathered}
\end{equation}
and,
\begin{equation}
\label{eq:expectation propagation solution cont}
\begin{gathered}
\tilde{\sigma}^2_i = 
\left(\hat{\sigma}_i^{-2} - \sigma_{-i}^{-2}\right)^{-1}, 
\quad
\tilde{\mu}_i = 
\tilde{\sigma}^2_i \left(\hat{\sigma}_i^{-2}\hat{\mu}_i-
\sigma_{-i}^{-2}\mu_{-i}\right) \\
\tilde{\eta}_i = 
\hat{\eta}_i
\sqrt{2\pi\left(\sigma^2_{-i}+\tilde{\sigma}^2_i\right)} 
\exp\left(
\frac{\left(\mu_{-i}-\tilde{\mu}_i\right)^2}
{2\left(\sigma^2_{-i}+\tilde{\sigma}^2_i\right)}
\right)
\end{gathered}
\end{equation}
In (\ref{eq:expectation propagation solution}), $\hat{\mu}_i$ and $\hat{\sigma}^2_i$ are the first and second moments of the distribution  $p_{-i}(m_{i})\Phi(y_i m_{i};0,1)$. A complete description of EP can be found in Algorithm~\ref{alg:expectation propagation}.

\begin{algorithm}[t]
\caption{Expectation Propagation}
\label{alg:expectation propagation}
\begin{algorithmic}[0]
\State Initialize $\tilde{\mu}_i = 0$, $\tilde{\sigma}_i=\infty$, and $\tilde{\eta}_i = 1$.
\Repeat
\ForAll{i}
\State Find the cavity parameters for $m_i$ with $q_{-i}(m_{i})$.
\State Update the site parameters for $t_i$ using \eqref{eq:expectation propagation solution} and \eqref{eq:expectation propagation solution cont}.
\State Update the mean and covariance for $q(\mathbf{m}|Z_{0:t})$.
\EndFor
\Until{convergence}
\end{algorithmic}
\end{algorithm}


\subsection{Accuracy}
\begin{theorem}
\label{thm:objective functions}
Given a prior of a random variable $\mathbf{m}$ with $\mathbf{m} \sim \mathcal{N}(\bm{\mu}, \Sigma)$ and a measurement $y_i\in\{-1,1\}$ with Bernoulli likelihood $y_i \sim \mathcal{B}(\Phi(y_i m_i;0,1))$. If the posterior of $\mathbf{m}$, i.e. $\phi(\mathbf{m}; \bm{\mu}, \Sigma)\Phi(y_i m_i;0,1)$ (the normalization factor is ignored here and onwards in this theorem), is to be approximated by another normal pdf $\mathbf{m} \sim \mathcal{N}(\hat{\bm{\mu}}, \hat{\Sigma})$, the following two approaches will give the same results.
\begin{equation}
\label{eq:context aware filter objective function}
\min_{\hat{\bm{\mu}}_c, \hat{\Sigma}_c} 
\text{KL}
\Big(
\phi(\mathbf{m}; \hat{\bm{\mu}}_c, \hat{\Sigma}_c) \,\big\|\,
\phi(\mathbf{m}; \bm{\mu}, \Sigma)\Phi(y_i m_i;0,1) 
\Big)
\end{equation}
and,
\begin{equation}
\label{eq:expectation propagation objective function}
\begin{aligned}
\min_{\tilde{\mu}_i, \tilde{\sigma}_i} 
\text{KL}
\Bigg(
&\int_{m_{j, j\neq i}}
\phi(\mathbf{m}; \bm{\mu}, \Sigma)
\phi(m_i; \tilde{\mu}_i, \tilde{\sigma}^2_i) dm_j \,\Big\|\, 
\\
&\int_{m_{j, j\neq i}}
\phi(\mathbf{m}; \bm{\mu}, \Sigma)
\Phi(y_i m_i;0, 1) dm_j
\Bigg)
\end{aligned}
\end{equation}
i.e.
\begin{equation*}
\phi(\mathbf{m}; \hat{\bm{\mu}}_c, \hat{\Sigma}_c) \sim 
\phi(\mathbf{m}; \bm{\mu}, \Sigma)
\phi(m_i; \tilde{\mu}_i, \tilde{\sigma}^2_i)
\end{equation*}
\end{theorem}
\begin{proof}
Recall that \eqref{eq:context aware filter objective function} is in the same form of \eqref{eq:KL divergence for context-aware filter}, which is the objective function of the occupancy grid filter at the measurement update step. Therefore, the solutions for $\hat{\bm{\mu}}_c$ and $\hat{\Sigma}_c$ follows the form of \eqref{eq:context aware filter discrete udpate}.

One critical observation for \eqref{eq:expectation propagation objective function} is that the approximated distribution $\phi(\mathbf{m}; \bm{\mu}, \Sigma)
\phi(m_i; \tilde{\mu}_i, \tilde{\sigma}^2_i)$ can be thought of as the product of a Gaussian prior and a measurement with a normal likelihood. The measurement model is this case is linear and can be represented as,
\begin{equation*}
h_i(\mathbf{m}) = \mathbf{v}_i^\top \mathbf{m} \sim \mathcal{N}(\tilde{\mu}_i, \tilde{\sigma}^2_i),
\end{equation*}
where $\mathbf{v}_i$ is defined similarly as in~\eqref{eq:context aware filter discrete udpate}. Then the resulting posterior distribution $\mathcal{N}(\hat{\bm{\mu}}_e, \hat{\Sigma}_e)$ can be readily computed by following the measurement update step of a Kalman filter,
\begin{equation}
\label{eq:kalman filter measurement update}
\begin{gathered}
K = \Sigma\mathbf{v}_i 
\left(\mathbf{v}_i^\top \Sigma \mathbf{v}_i + 
\tilde{\sigma}^2_i\right)^{-1} \\
\hat{\bm{\mu}}_e = \bm{\mu} + 
K\left(\tilde{\mu}_i-\mathbf{v}_i^\top\bm{\mu}\right) \\
\hat{\Sigma}_e = (I-K\mathbf{v}_i^\top) \Sigma
\end{gathered}
\end{equation}
Using the results for $\tilde{\mu}_i$ and $\tilde{\sigma}_i$ in~\eqref{eq:expectation propagation solution} and~\eqref{eq:expectation propagation solution cont},
\begin{equation*}
\hat{\bm{\mu}}_c = \hat{\bm{\mu}}_e, \quad 
\hat{\Sigma}_c = \hat{\Sigma}_e
\end{equation*}
i.e. the two optimization methods result in the same approximation for posterior of $\mathbf{m}$. The details for computing $\hat{\bm{\mu}}_e$ and $\hat{\Sigma}_e$ can be found in Appendix \ref{sec:complementary proof for theorem 1}.
\end{proof}

\begin{corollary}
The proposed method is a streamwise approximation of Expectation Propagation to solve Gaussian process occupancy mapping, where ``streamwise'' means processing the measurements sequentially without iterations.
\end{corollary}
\begin{proof}
We prove this corollary by applying Mathematical Induction. The initial prior for OGF is $\mathbf{m} \sim \mathcal{N}(\hat{\mathbf{m}}_0, \Sigma_0)$ where $\hat{\mathbf{m}}_0=\mathbf{0}$ and $\Sigma_0 = K$ is constructed by the kernels as in Gaussian processes. The initial approximation provided EP is also $\mathbf{m} \sim \mathcal{N}(\mathbf{0}, K)$, since $t_i=1$ for all $i$ given the initial settings in Algorithm~\ref{alg:expectation propagation}. Assume after processing the $i^\text{th}$ measurement at time step $t$, both OGF and EP have an approximation for the posterior of $\mathbf{m}$ with $\mathbf{m} \sim \mathcal{N}(\hat{\mathbf{m}}_t, \hat{\Sigma}_t)$. With the $(i+1)^\text{th}$ measurement at time $t+1$, the OGF and EP generates new approximations for the posterior with $\mathbf{m} \sim \mathcal{N}(\hat{\mathbf{m}}_{t+1, c}, \hat{\Sigma}_{t+1, c})$ and $\mathbf{m} \sim \mathcal{N}(\hat{\mathbf{m}}_{t+1, e}, \hat{\Sigma}_{t+1, e})$ by solving~\eqref{eq:KL divergence for EP} and~\eqref{eq:KL divergence for context-aware filter} respectively. Note the correspondence between~\eqref{eq:KL divergence for EP} and~\eqref{eq:expectation propagation objective function}, as well as~\eqref{eq:KL divergence for context-aware filter} and~\eqref{eq:context aware filter objective function}. A direct application of Theorem~\ref{thm:objective functions} shows that $\phi(\mathbf{m};\hat{\mathbf{m}}_{t+1, c}, \hat{\Sigma}_{t+1, c}) = \phi(\mathbf{m};\hat{\mathbf{m}}_{t+1, e}, \hat{\Sigma}_{t+1, e})$. Therefore, OGF is a streamwise approximation of the EP algorithm.
\end{proof} 

Conceptually, EP can be considered as a smoothing method requiring multiple iterations to converge. The proposed OGF is its corresponding filter approach which process each measurement only once. Assume both EP\footnote{As has been mentioned in~\cite{Rasmussen_2006}, EP is not guaranteed to converge. However, EP shows good convergence property in practical application.} and OGF converge in the sense that with new measurements, the estimated mean $\hat{\mathbf{m}}$ and uncertainty $\hat{\Sigma}$ no longer changes. Based on~\eqref{eq:context aware filter discrete udpate}, \eqref{eq:expectation propagation solution}, and \eqref{eq:expectation propagation solution cont}, this happens if and only if for every cell in the map, $\phi(m_i;0, 1) = 0$, which implies all cells are well classified. In practice, $m_i$ may converge to different values for EP and OGF because the gradient of the normal pdf approaches $0$ out of the three deviations. However, the difference should not affect the final classification result since $\Phi(m_i;0, 1)$ is almost at its extreme values when $\phi(m_i;0,1) \rightarrow 0$. This behavior will be more clear with the experiment results in the 2-D simulation environment in Sec.~\ref{subsec:2d simulation}.

\subsection{Complexity}
Representing the map size as $N$, we consider the computational complexity at time step $T$ with the assumption that a constant number of measurements, $M$, is taken at each time step. Then we have the following remark about the computation complexity.
\begin{remark}
For a single time step, the computation complexities for Laplace approximation, EP, and OGF are $\mathcal{O}(s_n M^3 T^3+N M^2 T^2)$, $\mathcal{O}(s_e M^3 T^3+N M^2 T^2)$, and $\mathcal{O}(M N^2)$ respectively.
\end{remark}
The cubic complexity of Laplace approximation comes from the inversion of the Hessian in Newton's method, and $s_n$ is the average number of iterations for Newton's method to converge. For EP, the cubic complexity results from the iterative update where the update for the site parameters of each $t_i$ has complexity $\mathcal{O}(M^2T^2)$. For OGF, no prediction step is required. The computation complexity only results from the measurement update. A direct observation is that, unlike conventional methods, the computation complexity of our approach does not scale with time. Also, in practical applications, the total number of measurements over time is usually in the same order with the dimension of the map. In this case, OGF has much lower computation complexity compared to Laplace approximation and EP.

\section{Experimental Results}
\label{sec:experimental results}
Two kinds of experimental results are shown in this section. In Section \ref{subsec:2d simulation}, we compare the accuracy of OGF with EP and the occupancy grid mapping algorithm~\cite{occgrid} in a 2-D simulated environment. In Section \ref{subsec:real world experiment}, we compare the maps built using OGF and Octomap~\cite{octomap} using 3-D point cloud measurements from a Velodyne. In both simulated and real world experiments, we used the standard normal pdf as the kernel function, i.e.
\begin{equation*}
k(\mathbf{x}, \mathbf{x}') = 
\phi(\|\mathbf{x}-\mathbf{x}'\|_2; 0, \sigma^2),
\quad
\mathbf{x},\mathbf{x}'\in\mathcal{X}
\end{equation*}
 
\subsection{2-D Simulation Environment}
\label{subsec:2d simulation}
We compare the proposed OGF with EP and occupancy grid mapping algorithm in a 2-D simulated environment with known ground truth. Noise-free samples were taken randomly from the map without repetition. The 2-D ground truth map and an example of $300$ random samples are shown in Figure~\ref{fig:2d simulation setup}.

\begin{figure}[htp]
\centering
\begin{subfigure}[b]{0.45\linewidth}
\includegraphics[width=\textwidth]{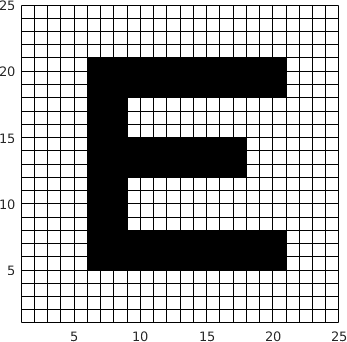}
\caption{•}
\label{subfig:2d ground truth map}
\end{subfigure}
\,
\begin{subfigure}[b]{0.45\linewidth}
\includegraphics[width=\textwidth]{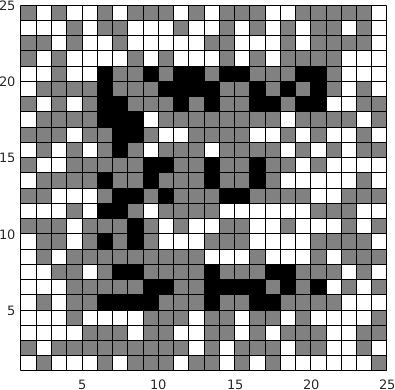}
\caption{•}
\label{subfig:300 sample example}
\end{subfigure}
\caption{(a) Ground truth map of dimension $25\times 25$ used in the 2-D simulation. (b) 300 samples were taken randomly from the map in (a) without repetition.}
\label{fig:2d simulation setup}
\end{figure}

\begin{figure*}[t]
\centering
\begin{subfigure}[b]{0.22\textwidth}
\includegraphics[width=\textwidth]{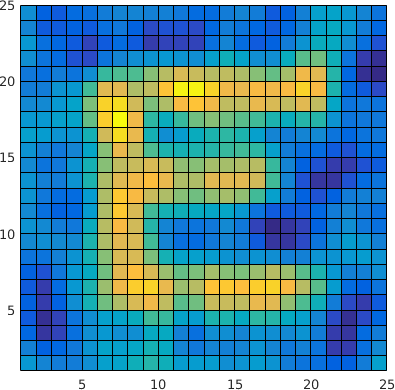}
\caption{•}
\end{subfigure}
\begin{subfigure}[b]{0.22\textwidth}
\includegraphics[width=\textwidth]{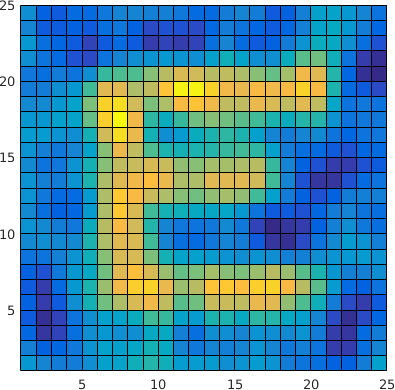}
\caption{•}
\end{subfigure}
\begin{subfigure}[b]{0.22\textwidth}
\includegraphics[width=\textwidth]{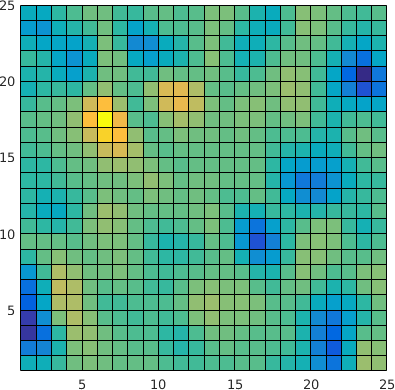}
\caption{•}
\label{subfig:map diff with 300 samples}
\end{subfigure}
\begin{subfigure}[b]{0.25\textwidth}
\includegraphics[width=\textwidth]{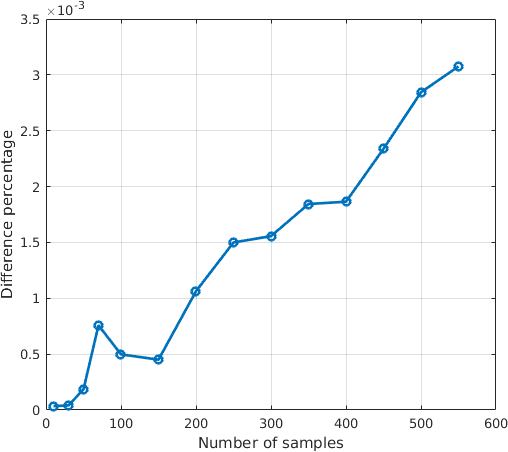}
\caption{•}
\label{subfig:map difference vs sample number}
\end{subfigure}
\caption{(a) Map difference percentage $d(\hat{\mathbf{m}}_{c}, \hat{\mathbf{m}}_{e})$, as defined in \eqref{eq:map difference metric}, with respect to the number of samples. (b) Map estimated with EP. (c) Map estimated OGF. (d) Difference between the two estimated maps. Both of the maps were estimated using the $300$ training samples shown in Figure \ref{subfig:300 sample example} with the standard deviation of the normal pdf kernel set to $1$ grid unit.}
\label{fig:estimated map and map diff with 300 samples}
\end{figure*}

\begin{figure*}[t]
\centering
\begin{subfigure}[b]{0.25\textwidth}
\includegraphics[width=\textwidth]{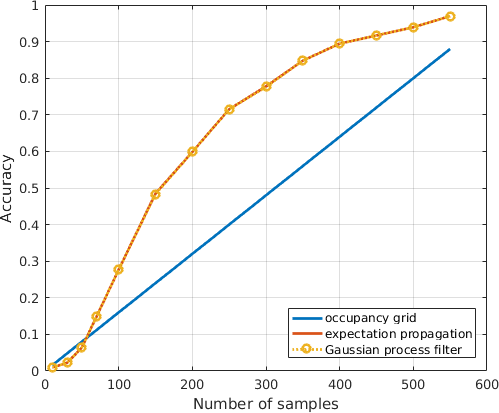}
\caption{•}
\label{subfig:accuracy comparison between cf, ep and og}
\end{subfigure}
\begin{subfigure}[b]{0.22\textwidth}
\includegraphics[width=\textwidth]{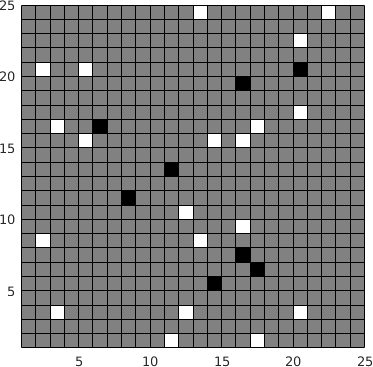}
\caption{•}
\end{subfigure}
\begin{subfigure}[b]{0.22\textwidth}
\includegraphics[width=\textwidth]{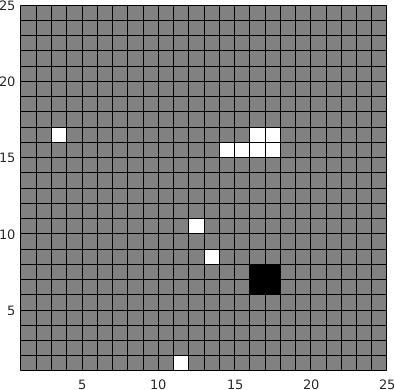}
\caption{•}
\end{subfigure}
\begin{subfigure}[b]{0.22\textwidth}
\includegraphics[width=\textwidth]{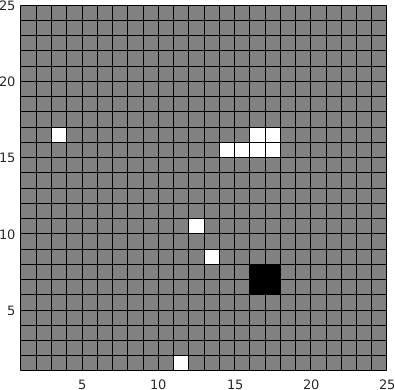}
\caption{•}
\end{subfigure}
\\
\begin{subfigure}[b]{0.22\textwidth}
\includegraphics[width=\textwidth]{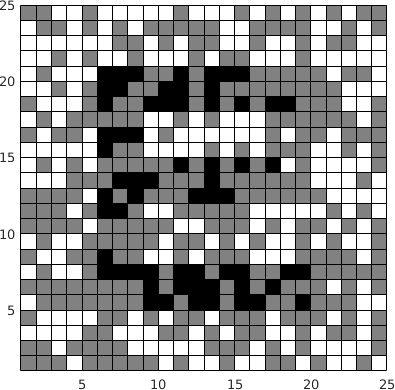}
\caption{•}
\end{subfigure}
\begin{subfigure}[b]{0.22\textwidth}
\includegraphics[width=\textwidth]{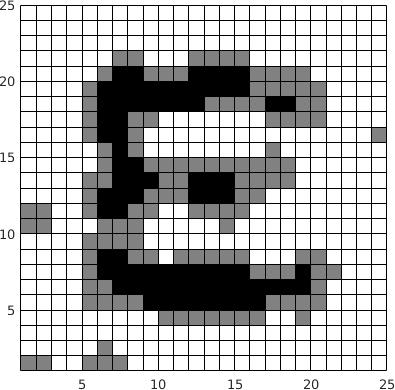}
\caption{•}
\end{subfigure}
\begin{subfigure}[b]{0.22\textwidth}
\includegraphics[width=\textwidth]{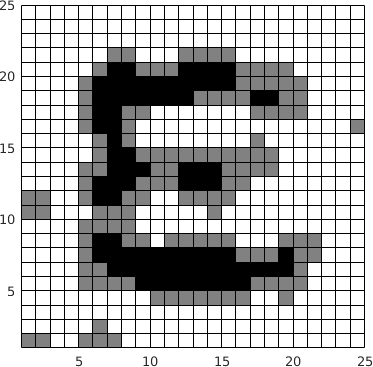}
\caption{•}
\end{subfigure}
\caption{(a) The accuracy, as defined in \eqref{eq:classification accuracy}, of the map generated by the occupancy grid mapping algorithm, EP, and OGF. (b), (c), and (d) are the maps generated by the three algorithms with $30$ samples. (e), (f), and (g) are the maps generated by the same algorithms with $300$ samples. For EP and OGF, $r_o$ and $r_f$ were set to $0.65$ and $0.35$ respectively.}
\label{fig:map of cf, ep and og with 30 and 300 samples}
\end{figure*}

Figure~\ref{fig:estimated map and map diff with 300 samples} shows the map estimated by both EP and OGF, as well as the difference between the two maps, with the $300$ measurements shown in Figure \ref{subfig:300 sample example}. As can be observed directly in Figure \ref{subfig:map diff with 300 samples}, the differences were mainly concentrated at the well-classified regions where the absolute latent values are relatively large. This echoes our analysis in Section \ref{sec:comparison}. 

We used~\eqref{eq:map difference metric} to convert the difference of the latent values into a scalar metric, which is the difference as a percentage of the map estimated by EP.
\begin{equation}
\label{eq:map difference metric}
d(\hat{\mathbf{m}}_{c}, \hat{\mathbf{m}}_{e}) = 
\frac{\|\hat{\mathbf{m}}_{c}-\hat{\mathbf{m}}_{e}\|_2}{\|\hat{\mathbf{m}}_{e}\|_2}
\end{equation}
Figure~\ref{subfig:map difference vs sample number} shows $d(\hat{\mathbf{m}}_{c}, \hat{\mathbf{m}}_{e})$ with respect to the number of measurements. Based on the plot, the difference $d(\hat{\mathbf{m}}_{c}, \hat{\mathbf{m}}_{e})$ was generally less than $0.04\|\hat{\mathbf{m}}_{e}\|_2$. It can also be observed that $d(\hat{\mathbf{m}}_{c}, \hat{\mathbf{m}}_{e})$ increased with the number of measurements, which was due to the fact that more cells in the map should be well-classified as the measurement increases. As shown in the later examples, this minor error in latent values does not affect the classification results.

Figure \ref{subfig:accuracy comparison between cf, ep and og} shows the accuracy of the map generated by the occupancy grid mapping algorithm, EP, and OGF with the accuracy defined as,
\begin{equation}
\label{eq:classification accuracy}
acc = \frac{1}{|\mathcal{X}|} 
\sum_i \mathbf{1}(\hat{o}(\mathbf{x}_i) = o(\mathbf{x}_i)),
\quad
\mathbf{x}_i\in\mathcal{X}
\end{equation}
where $|\mathcal{X}|$ is the map size, while $o(\mathbf{x})$ and $\hat{o}(\mathbf{x})$ are defined in Sec.~\ref{sec:problem formulation} and~\ref{sec:gaussian process}. It should be noted that occupancy grid mapping was trivial in this case; since the measurements were noise-free, the map generated by the occupancy grid algorithm was the same as marking sample positions as whatever the measurements were.

As can be seen from Figure~\ref{subfig:accuracy comparison between cf, ep and og}, the map generated using the proposed method has the same accuracy as EP for various numbers of measurements. This supports our previous claim that the minor difference in latent value estimation does not affect the classification. It is also indicated in Figure~\ref{subfig:accuracy comparison between cf, ep and og} that both EP and OGF have lower accuracy than the occupancy grid mapping when there are few samples, but outperform it once the samples get reasonably dense. The reason for this is that sparse measurements will not alter the latent value of a cell far from its prior mean so the occupancy status of most of the cells remain unknown. However, when the density of samples increases the OGF is able to infer the occupancy of the cells in the neighborhood of the measured samples. The required density depends on the type and the parameters of the kernel in use. Figure~\ref{fig:map of cf, ep and og with 30 and 300 samples} shows the map generated by the three algorithms with $30$ and $300$ samples respectively. 

\subsection{Real World Data}
\label{subsec:real world experiment}
\begin{figure*}[t]
\centering
\begin{subfigure}[b]{0.32\textwidth}
\includegraphics[width=\textwidth]{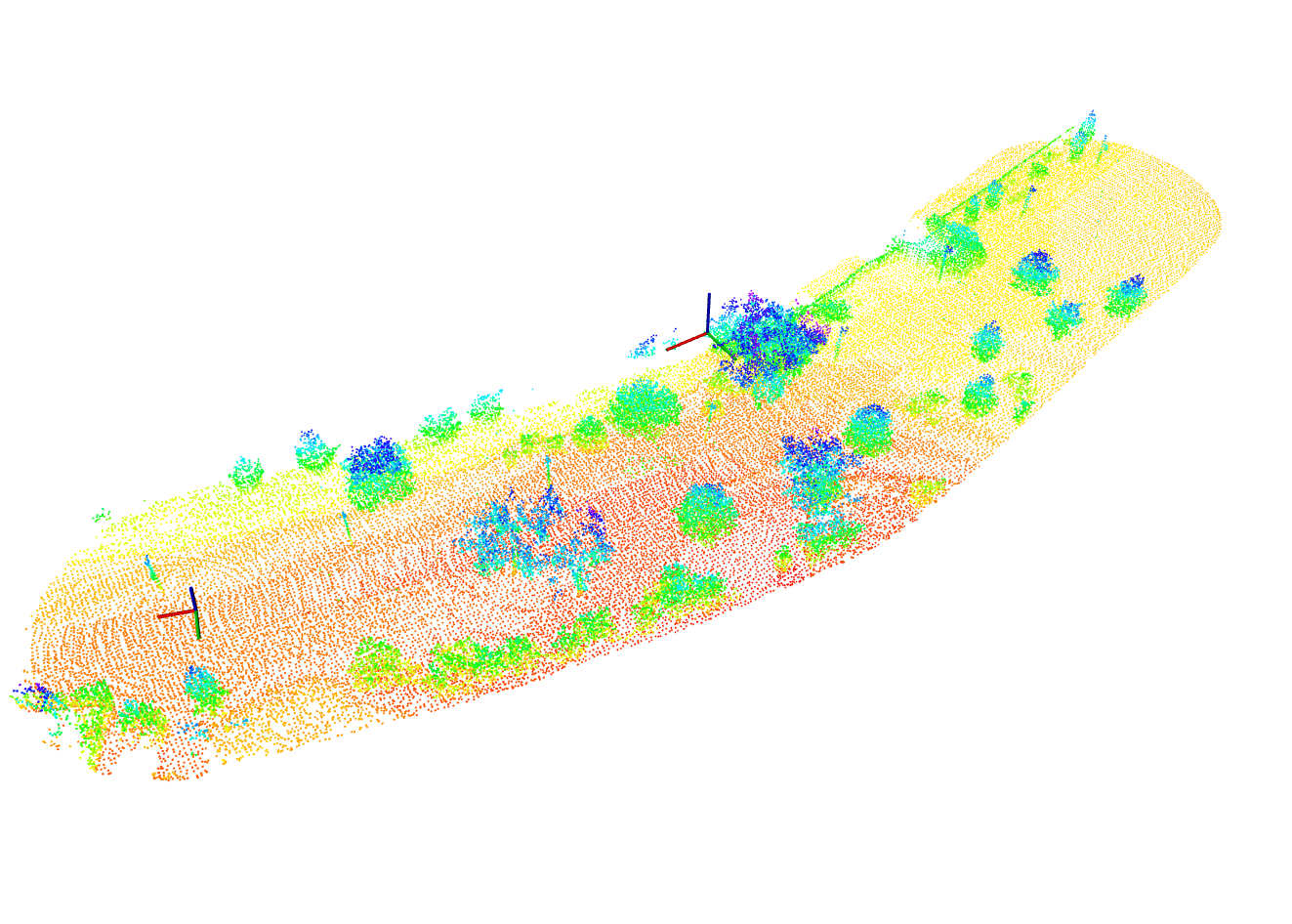}
\caption{•}
\end{subfigure}
\begin{subfigure}[b]{0.32\textwidth}
\includegraphics[width=\textwidth]{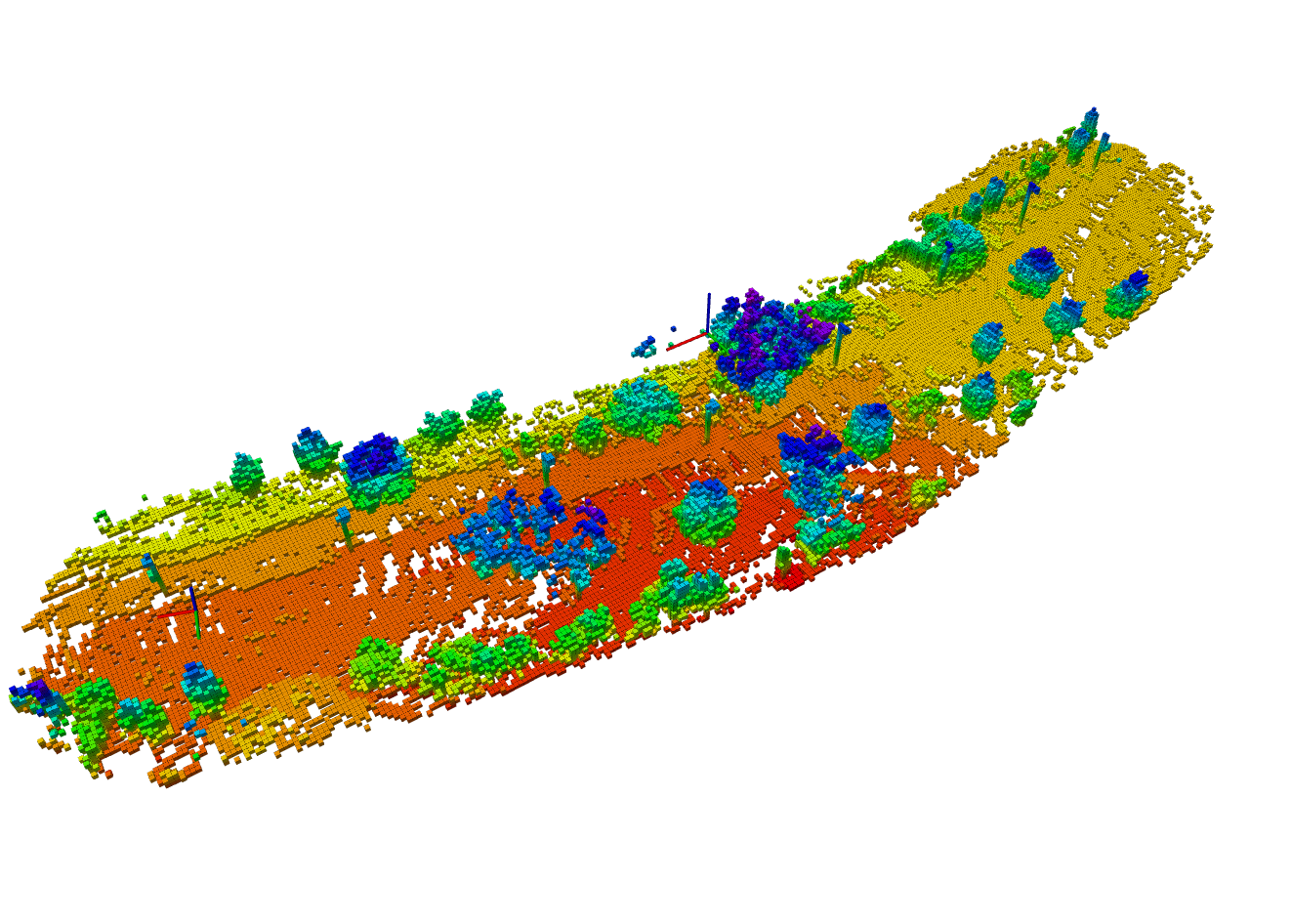}
\caption{•}
\end{subfigure}
\begin{subfigure}[b]{0.32\textwidth}
\includegraphics[width=\textwidth]{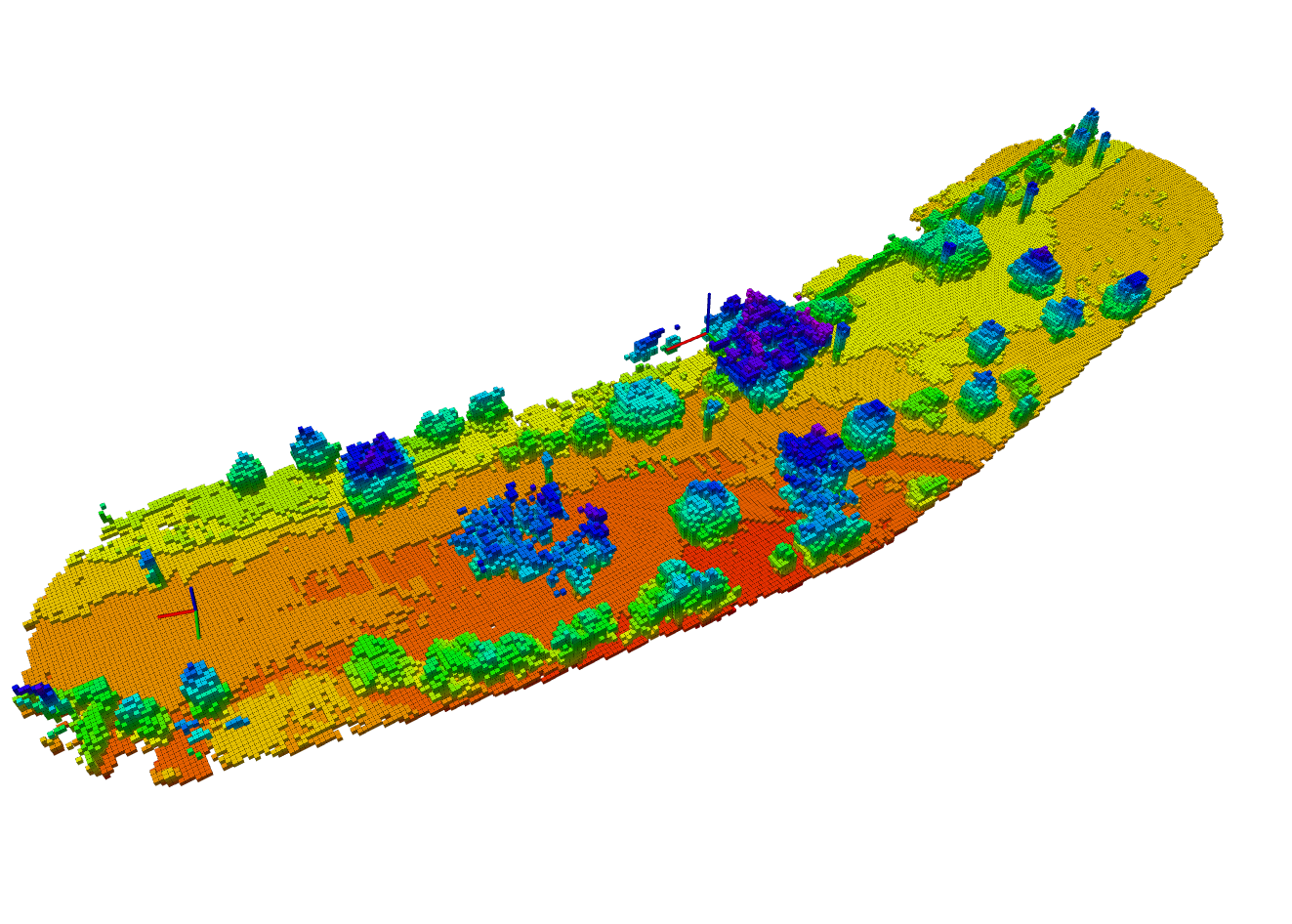}
\caption{•}
\end{subfigure}
\\
\begin{subfigure}[b]{0.32\textwidth}
\includegraphics[width=\textwidth]{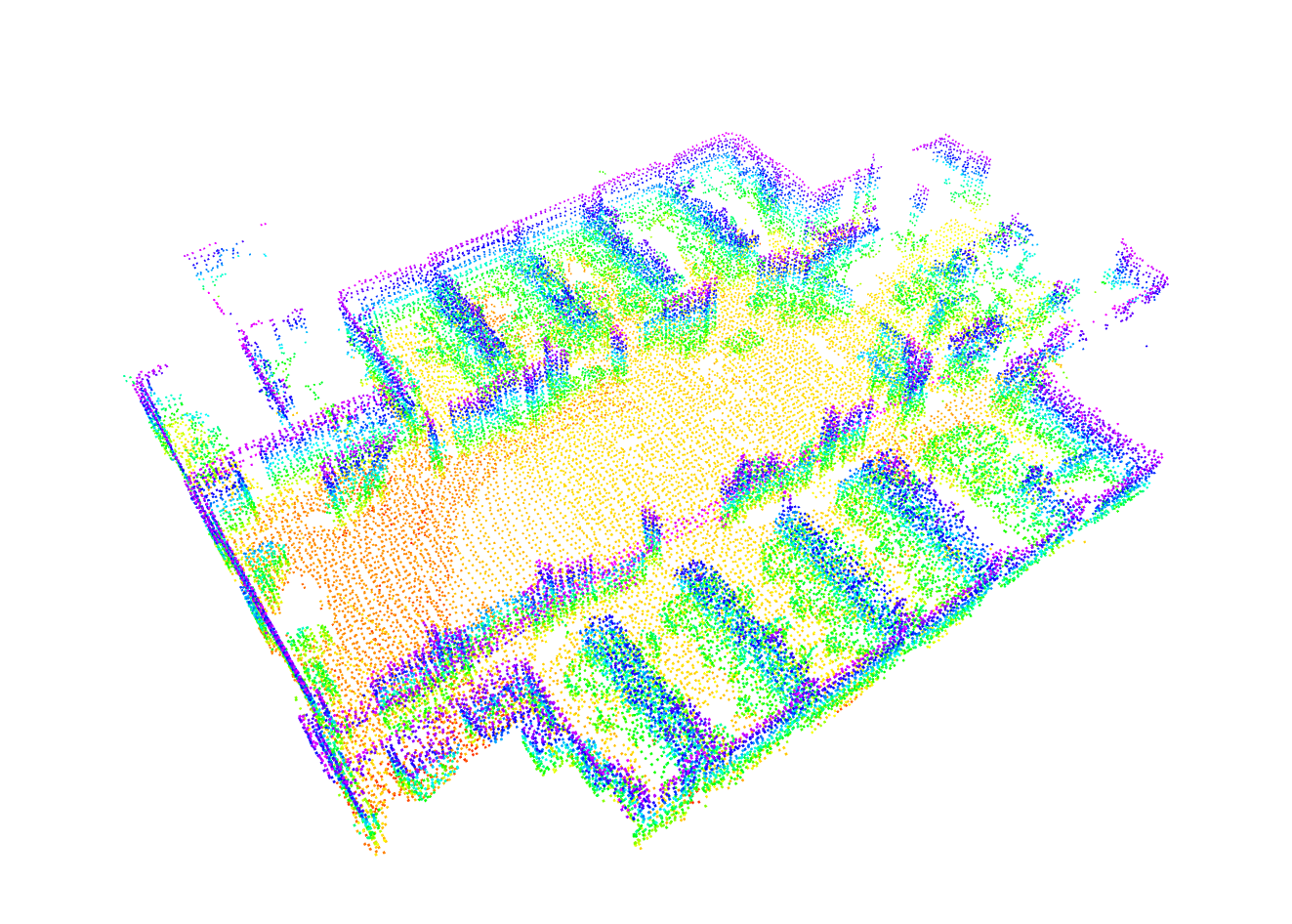}
\caption{•}
\end{subfigure}
\begin{subfigure}[b]{0.32\textwidth}
\includegraphics[width=\textwidth]{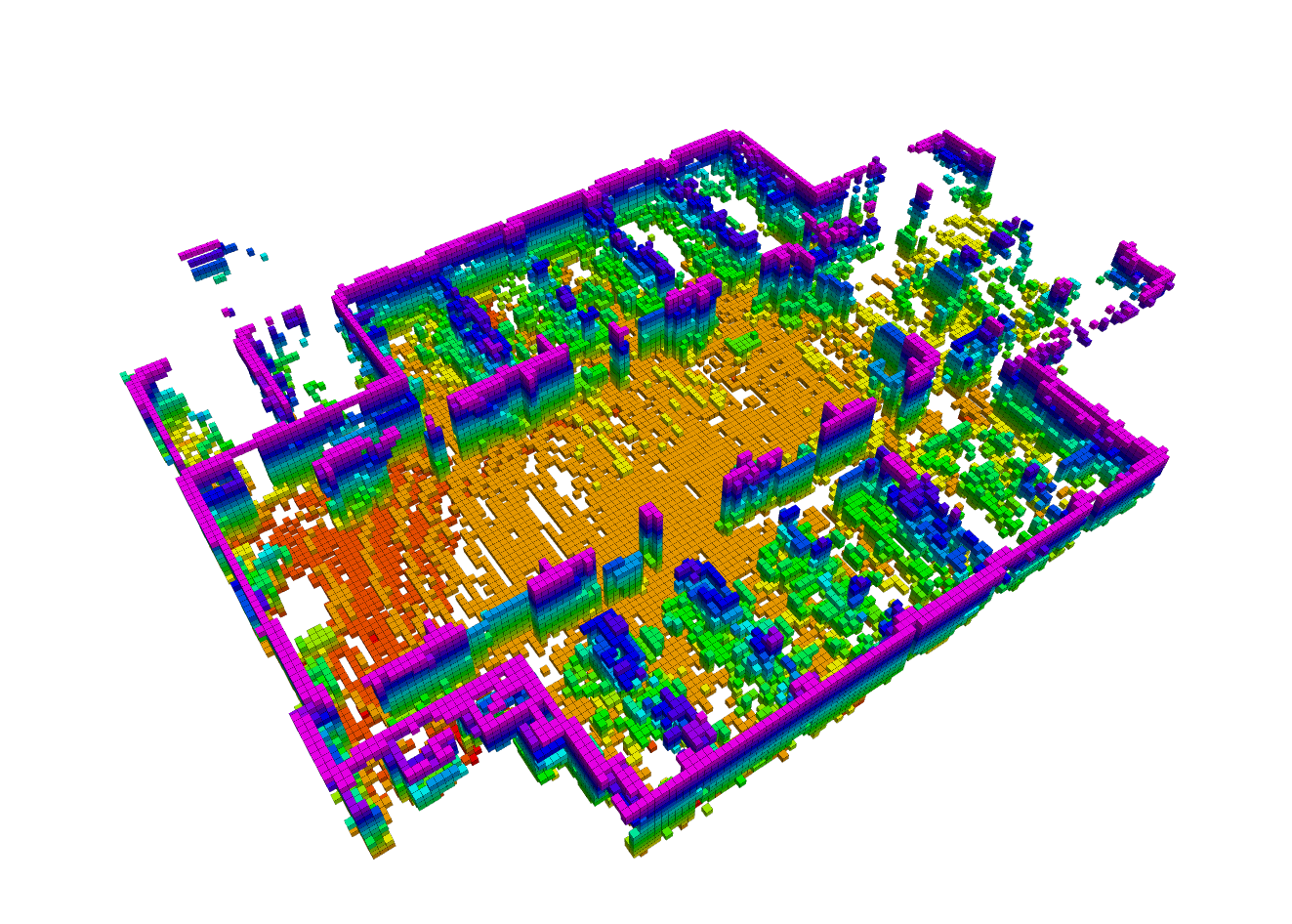}
\caption{•}
\end{subfigure}
\begin{subfigure}[b]{0.32\textwidth}
\includegraphics[width=\textwidth]{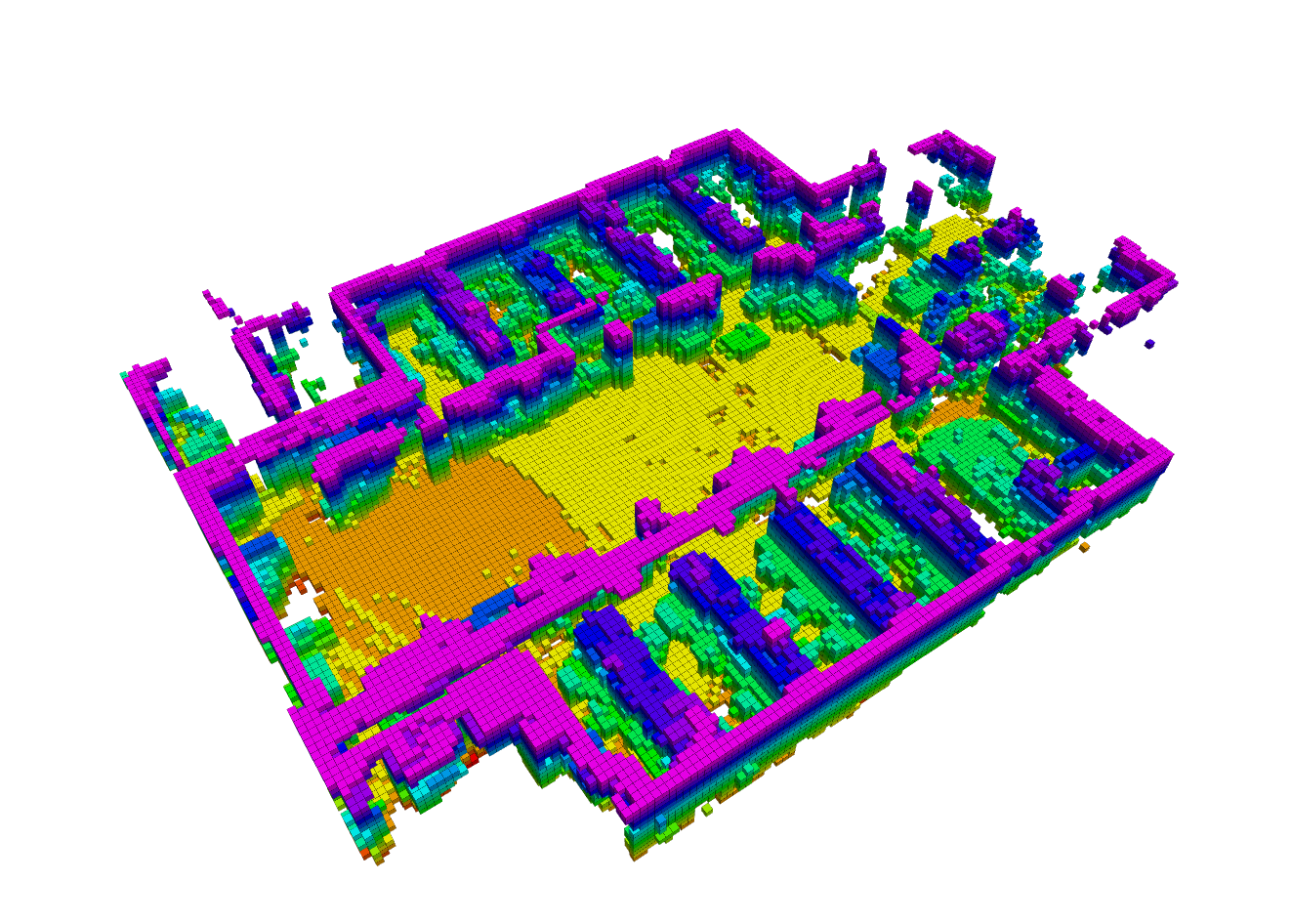}
\caption{•}
\end{subfigure}
\caption{The first column is the raw point cloud directly registered using the Velodyne measurements. The point cloud was filtered with the PCL voxel grid filter with $0.5m$ leaf size. The second and third columns are the occupied cells in the map generated by the Octomap and OGF respectively. The dataset of the first row was taken in a outdoor environment with dimension $150m\times 60m \times 8m$. In the outdoor dataset, the cell size used for both Octomap and the OGF was $0.4m$. The second row shows a dataset from an indoor environment with dimension $35m\times 30m\times 2m$ and cell size $0.2m$. In both datasets, the standard deviation used in the kernel function was defined as half the size of the corresponding cell, i.e. $0.2m$ and $0.1m$. The threshold values, $r_o$ and $r_f$, were set to $0.65$ and $0.35$ respectively.}
\label{fig:comparison between cf and octomap}
\end{figure*}


In order to evaluate the effectiveness of OGF with real world data we implemented the proposed algorithm using C++ with ROS\footnote{www.ros.org}. A Velodyne Puck (VLP-16)\footnote{velodynelidar.com/vlp-16.html} was used to provide 3-D point cloud measurements. The lidar-based odometry method (LOAM)~\cite{zhang2014loam} was used to estimate the position and orientation of the sensor.

To enable the algorithm to deal with the dense point clouds provided by Velodyne, we applied the following downsampling method that reduced the number of samples taken at each measurement update step, but at the same time kept all the samples uniformly distributed across the 3-D space. For each lidar ray, we used the ray tracing algorithm introduced in~\cite{amanatides1987fast} to identify which cells that ray has travelled through. A measurement was only taken in a cell if no measurement had been taken within this cell before. As has been mentioned in \cite{octomap}, a cell measured as free by one lidar ray may be reported as occupied in others due to the discretization of the 3-D space. This effect is especially obvious when the map resolution is coarse. To overcome this problem, we prioritized occupied measurements of a cell over free measurements of the same cell, i.e a cell previously measured as free could be re-measured as occupied, but the opposite was not allowed.

Figure~\ref{fig:comparison between cf and octomap} shows a comparison between the map generated by Octomap and OGF in both outdoor and indoor environments. In the experiment with real world data, OGF demonstrated its capability of generating a denser map than Octomap by filling the gaps, which is especially obvious comparing the reconstructed ground.

\section{Conclusion}
\label{sec:conclusion}
This paper presented an incremental approach for solving the occupancy estimation problem. We proved that the proposed Occupancy Grid Filter is a streamwise approximation of EP, one of the conventional methods to solve GP classification, with constant computation complexity in time. With the 2-D simulation environment, we demonstrated that OGF generates almost the same classification results as EP, albeit with minor differences in the latent values only for the well-classified cells. Also, the accuracy outperforms the occupancy grid mapping algorithm with reasonable number of measurements. We also evaluated the performance of OGF on real world data, and compared the results with Octomap. It was demonstrated that, because of the more accurate representation of the uncertainty of the map, OGF is able to fill the gaps and produce a map that is reasonably denser.

As the computation complexity of our method is quadratic in the map size, for the future work, we would like to explore the possibility of developing a map representation with an adaptive resolution that is compatible with the current framework. Furthermore, we would like to understand if the covariance information provided by OGF, which drops the independence assumption between the cells in the grid, can be a better guidance for autonomous exploration and mapping compared to the classical occupancy grid map.

\appendices
\section{Complementary proof for Theorem 1}
\label{sec:complementary proof for theorem 1}
In this section, we complete the intermediate steps of \eqref{eq:kalman filter measurement update}, which is the measurement update step of a Kalman filter with prior $\mathbf{m} \sim \mathcal{N}(\bm{\mu}, \Sigma)$ and a linear measurement model $h_i(\mathbf{m}) = \mathbf{v}_i^\top \mathbf{m} \sim \mathcal{N}(\tilde{\mu}_i, \tilde{\sigma}^2_i)$. Define,
\begin{equation*}
\begin{gathered}
\Delta \bm{\mu}_e \coloneqq \hat{\bm{\mu}}_e - \bm{\mu},
\qquad
\Delta \Sigma_e \coloneqq \hat{\Sigma}_e - \Sigma
\end{gathered}
\end{equation*}
Then change of mean is,
\begin{equation*}
\Delta \bm{\mu}_e
= \Sigma \mathbf{v}_i \cdot 
\frac{\tilde{\mu}_i-\mu_i}{\sigma_i^2+\tilde{\sigma}^2_i}
\end{equation*}
Using the results from \eqref{eq:expectation propagation solution} and \eqref{eq:expectation propagation solution cont}, $\Delta \bm{\mu}_e$ can be computed as in \eqref{eq:Kalman filter update of mean}. Note that $\mu_i$ and $\sigma_i$ are used in \eqref{eq:Kalman filter update of mean} to represent the cavity parameters instead of $\mu_{-i}$ and $\sigma_{-i}$. 
\begin{equation}
\label{eq:Kalman filter update of mean}
\begin{aligned}
\Delta \bm{\mu}_e
&= \Sigma \mathbf{v}_i \cdot 
\frac{\tilde{\sigma}^2_i 
\left(\hat{\sigma}_i^{-2}\hat{\mu}_i -\sigma^{-2}_i\mu_i\right) - \mu_i}
{\sigma_i^2+\tilde{\sigma}^2_i} \\
&= \Sigma \mathbf{v}_i \cdot 
\frac{\hat{\sigma}_i^{-2}\hat{\mu}_i -\sigma^{-2}_i\mu_i - \tilde{\sigma}^{-2}_i\mu_i}
{\sigma_i^2 \tilde{\sigma}^{-2}_i+1} \\
&= \Sigma \mathbf{v}_i \cdot 
\frac{\hat{\sigma}_i^{-2}\hat{\mu}_i - \left(\sigma^{-2}_i + \tilde{\sigma}^{-2}_i\right) \mu_i}
{\sigma_i^2 \left(\hat{\sigma}_i^{-2}-\sigma_i^{-2}\right) +1} \\
&= \Sigma \mathbf{v}_i \cdot 
\frac{\hat{\sigma}_i^{-2}\hat{\mu}_i - \hat{\sigma}_i^{-2}\mu_i}
{\sigma_i^2\hat{\sigma}_i^{-2}} \\
&= \Sigma \mathbf{v}_i \cdot 
\frac{\hat{\mu}_i - \mu_i}{\sigma_i^2} \\
&= \Sigma \mathbf{v}_i \cdot 
\frac{y_i \phi(z_i;0, 1)}{\Phi(z_i;0,1)\sqrt{1+\sigma_i^2}}
\end{aligned} 
\end{equation}
Similarly, the change of covariance is,
\begin{equation*}
\label{eq:Kalman filter update of covariance}
\begin{aligned}
\Delta \Sigma_e
&= -\Sigma \mathbf{v}_i \mathbf{v}_i^\top \Sigma \cdot 
\frac{1}{\sigma_i^2+\tilde{\sigma}^2_i} \\
&= -\Sigma \mathbf{v}_i \mathbf{v}_i^\top \Sigma \cdot 
\frac{\tilde{\sigma}^{-2}_i}{\sigma_i^2 \tilde{\sigma}^{-2}_i + 1} \\
&= -\Sigma \mathbf{v}_i \mathbf{v}_i^\top \Sigma \cdot 
\frac{\hat{\sigma}_i^{-2}-\sigma_i^{-2}}
{\sigma_i^2 \left(\hat{\sigma}_i^{-2}-\sigma_i^{-2}\right) + 1} \\
&= -\Sigma \mathbf{v}_i \mathbf{v}_i^\top \Sigma \cdot 
\frac{\hat{\sigma}_i^{-2}-\sigma_i^{-2}}
{\sigma_i^2 \hat{\sigma}_i^{-2}} \\
&= -\Sigma \mathbf{v}_i \mathbf{v}_i^\top \Sigma \cdot 
\frac{\phi(z_i;0,1)}{(1+\sigma_i^2)\Phi(z_i;0,1)} 
\left(z_i+\frac{\phi(z_i;0,1)}{\Phi(z_i;0,1)}\right) \\
&= -\Sigma \mathbf{v}_i \mathbf{v}_i^\top \Sigma \cdot 
\frac{y_i\mu_i\cdot\phi(z_i;0,1)}{(1+\sigma_i^2)^{3/2}\Phi(z_i;0,1)} - 
\Delta \bm{\mu}_e \Delta\bm{\mu}_e^\top
\end{aligned}
\end{equation*}
The change of mean and covariance matches the measurement update step in the occupancy grid filter in \eqref{eq:context aware filter discrete udpate}.

\bibliographystyle{IEEEtran}
\bibliography{ref.bib}

\end{document}